\documentclass[twocolumn]{article}
\usepackage{amsmath,amssymb,amsbsy,amsthm,algorithmic,algorithm, natbib}

\let\epsilon\varepsilon
\let\phi\varphi

\def\N{\mathbb N}
\def\S{\mathcal S}
\def\R{\mathbb R}

\def\E{{\bf E}}

\def\x{{\mathbf{x}}}

\def\-as{\text{-a.s.}}
\def\argmax{\operatorname{argmax}}
\def\argmin{\operatorname{argmin}}

\newtheorem{theorem}{Theorem}
\newtheorem{definition}{Definition}
\newtheorem{lemma}{Lemma}

\newtheorem{proposition}{Proposition}

\begin{document}
\title{Clustering processes}
\author{Daniil Ryabko \\ {\em INRIA Lille-Nord Europe,} \\ {\tt daniil@ryabko.net}}
\date{}
\maketitle

\begin{abstract} The problem of clustering is considered, for the case when each data point
is a sample generated by a stationary ergodic process. We propose a very natural asymptotic notion of consistency,
and show that simple consistent  algorithms exist, under most general non-parametric assumptions. 
The notion of consistency is as follows: two samples should be put into the same cluster if and only if they were
generated by the same distribution. With this notion of consistency, clustering generalizes such classical
statistical problems as homogeneity testing and process classification.  We show that, for the case of a known number of clusters, consistency can be achieved 
under the only assumption that the joint distribution of the data is stationary ergodic (no parametric or Markovian assumptions, 
no assumptions of independence, neither between nor within the samples). If the number of clusters is unknown, 
consistency can be achieved under appropriate assumptions on the mixing rates of the processes. 
In both cases we give examples of simple (at most quadratic in each argument) algorithms
which are consistent. 
\end{abstract}

\section{Introduction}
Given a finite set of objects, the problem is to ``cluster'' similar objects together. 
This intuitively simple goal is notoriously hard to formalize.
Most of the work on clustering is concerned with particular parametric data generating models, or particular algorithms,
a given  similarity measure, and (very often) a given number of clusters.
It is clear that, as in almost  learning problems, in clustering finding the right similarity measure is an integral 
part of the problem. However, even if one assumes  the similarity 
measure known, it is hard to define what a good clustering is \cite{Kleinberg:02, Zadeh:09}. 
What is more, even if one assumes the similarity measure to be simply the Euclidean distance (on the plane), 
and the number of clusters  $k$ known, then clustering may still appear intractable for computational reasons.
Indeed, in this case finding $k$ centres (points which minimize the cumulative distance from each point in the sample to one of the centres)
 seems to be a natural goal, but this problem is NP-hard \cite{Mahajan:09}.

In this work we concentrate on a subset of the clustering problem: clustering processes. That is, 
each data point is itself a sample generated by a certain discrete-time stochastic process. 
This version of the problem has numerous applications, such as clustering biological data, financial
observations, or behavioural patterns, and as such it has gained a tremendous attention in the literature.

The main observation that we make in this work is that, in the case of clustering processes,  one can 
 benefit from the notion of {\em ergodicity} to define what appears to be a very natural notion of consistency.
This notion of consistency is shown to be satisfied by simple algorithms that we present, which are polynomial in all
arguments. This can be achieved without any modeling assumptions on the data (e.g. Hidden Markov, Gaussian, etc.), without assuming 
 independence of any kind within or between   the samples. 
The only assumption that we make is that the joint distribution of the data is stationary ergodic. 
The assumption of stationarity means, intuitively, that the time index itself bares no information:
it does not matter whether we have started recording observations at time 0 or at time 100. By virtue 
of the ergodic theorem, any stationary process can be represented as a mixture of stationary ergodic processes. 
In other words,  a stationary process can be thought of as first selecting a stationary ergodic
process (according to some prior distribution) and then observing its outcomes. Thus, the assumption 
that the data is stationary ergodic is both very natural and rather weak. At the same time, 
ergodicity means that, in asymptotic, the  properties of the process can be learned from observation.

This allows us to define the clustering problem as follows.  $N$ samples are given: $\x_1=(x^1_1,\dots,x^1_{n_1}),\dots, \x_N=(x^N_1,\dots,x^N_{n_N})$.
Each sample is drawn by one out of $k$ different stationary ergodic distributions. The samples are {\em not} assumed to be drawn 
independently; rather, it is assumed that the joint distribution of the samples is stationary ergodic. The target clustering is as follows:
those and only those samples are put into the same cluster  that were generated by the same distribution. 
The number $k$ of target clusters can be either known or unknown (different consistency results can be obtained in these cases).
A clustering algorithm is called asymptotically consistent if the probability that it outputs the target clustering converges to 1, 
as the lengths ($n_1,\dots,n_N$) of the samples tend to infinity (a variant of this definition is to require the algorithm 
to stabilize on the correct answer with probability 1). Note the particular regime of asymptotic: not with respect
to the number of samples $N$, but with respect to the length of the samples $n_1,\dots,n_N$.

Similar formulations have appeared in the literature before. Perhaps the most close approach is mixture models \cite{Smyth:97, Zhong:03}:
it is assumed that there are $k$ different distributions that have a particular known form (such as Gaussian, Hidden Markov models, 
or graphical models) and each one out of $N$ samples is generated independently according to one of these 
$k$ distributions (with some fixed probability). Since the model of the data is specified quite well, one can use  likelihood-based distances
 (and then, for example, the $k$-means algorithm), or Bayesian
inference, to cluster the data.  Clearly, the main difference from our setting is in that we do not assume any known model of the data; not even between-sample
independence is assumed.

The problem of clustering in our formulation generalizes two classical problems of mathematical statistics.
The first one is homogeneity testing, or the two-sample problem. Two samples $\x_1=(x^1_1,\dots,x^1_{n_1})$ and $\x_2=(x^2_1,\dots,x^2_{n_2})$ are given, 
and it is required to test whether they 
were generated by the same distribution, or by different distributions. This corresponds to clustering 
just two data points ($N=2$) with the number $k$ of clusters unknown: either $k=1$ or $k=2$. 
The second problem is process classification, or the three-sample problem. Three samples $\x_1,\x_2,\x_3$ are given, it is known
that two of them were generated by the same distribution, while the third one was generated by a different distribution.
It is required to find out which two were generated by the same distribution. 
This corresponds
to clustering three data points, with the number of clusters known: $k=2$.
The classical approach is of course to consider Gaussian i.i.d. data, but general non-parametric solutions 
exist not only for i.i.d. data \cite{Lehmann:86}, but also for Markov chains \cite{Gutman:89}, and under certain mixing rates conditions. 
What is important for us here, is that the three-sample problem is easier than the two-sample problem;
the  reason is that $k$ is known in the latter case but not in the former. Indeed, in \cite{Ryabko:10discr} it is shown
that in general, for stationary ergodic (binary-valued) processes, there is no solution to the two-sample problem, 
even in the weakest asymptotic sense. However, a solution to the three-sample problem, for (real-valued) stationary 
ergodic processes was given in \cite{Ryabko:103s}.

In this work we demonstrate that, if the number $k$ of clusters is known, then there is an asymptotically 
consistent clustering algorithm, under the only assumption that the joint distribution of data is stationary ergodic. 
If $k$ is unknown, then in this general case there is no consistent clustering algorithm (as follows from the mentioned
result for the two-sample problem). However, if an upper-bound $\alpha_n$ on the $\alpha$-mixing rates of the joint distribution of the processes
is known, and $\alpha_n\to0$, then there  is a consistent clustering algorithm. 
Both algorithms are rather simple, and are based on the empirical estimates of the so-called distributional distance.
 For two processes $\rho_1, \rho_2$
a distributional distance  $d$ is defined as $\sum_{k=1}^\infty w_k |\rho_1(B_k)-\rho_2(B_k)|$,
where $w_k$ are positive summable real weights, e.g. $w_k=2^{-k}$, and $B_k$ range
over a countable field that generates the sigma-algebra of the underlying probability space.
For example, if we are talking about finite-alphabet processes with the binary alphabet $A=\{0,1\}$, $B_k$ would
range over the set $A^*=\cup_{k\in\N} A^k$; that is, over all tuples $0, 1, 00, 01, 10, 11, 000, 001,\dots$ (of course, we could just as well omit, say, $1$ and $11$); therefore,
the distributional distance in this case is the weighted sum of differences of probabilities of all possible tuples.
In this work we consider real-valued processes,  so $B_k$ have to range through a suitable sequence of intervals, all pairs of such intervals, triples, etc. (see
the formal definitions below). This distance has proved a useful tool for solving various statistical problems 
concerning ergodic processes \cite{Ryabko:103s, Ryabko:101c}.

Although this distance involves infinite summation, we show that its empirical approximations can be easily calculated.
For the case of a known number of clusters, the proposed algorithm (which is shown to be consistent) is as follows.
(The distance in the algorithms is a suitable empirical estimate of $d$.)
The  first sample  is assigned to the first cluster. 
 For each $j=2..k$, find a point that maximizes the minimal distance
to those points already assigned to clusters, and assign it to the cluster $j$.
Thus we have one point in each of the $k$ clusters. Next, assign each of the remaining points
to the cluster that contains the closest points from those $k$ already assigned.
For the case of an unknown number of clusters $k$, the algorithm simply puts those samples together 
that are not farther away from each other than a certain threshold level, where the threshold is calculated based
on the known bound on the mixing rates. In this case, besides the asymptotic result, finite-time bounds on the probability 
of outputting an incorrect clustering can be obtained.
Each of the algorithms is shown to be at most quadratic in each argument.

Therefore, we show that for the proposed notion of consistency, there are simple algorithms that are
consistent under most general assumptions. While these algorithms can be easily implemented, we have 
left the problem of trying them out on particular applications, as well as optimizing the parameters,
for future research. It may also be suggested that the empirical distributional distance can be 
replaced by other distances, for which similar theoretical results can be obtained.
An interesting direction, that could preserve the theoretical generality,  would be to use data compressors. These were used in 
\cite{BRyabko:06a} 
for the related problems of hypotheses testing, 
 leading both to theoretical and practical
results. As far as clustering is concerned, compression-based methods were used (without asymptotic consistency analysis) in \cite{Cilibrasi:05}, and (in a different way)
in \cite{Bagnall:05}.  Combining our consistency framework with these compression-based
methods is a promising direction for further research.

\section{Preliminaries}\label{s:pre}
Let $A$ be an alphabet, and denote  $A^*$ the set of  tuples $\cup_{i=1}^\infty A^i$. 
In this work we consider the case  $A=\R$; 
extensions to the multidimensional case, as well as to more general spaces, are
straightforward. 
 Distributions, or  (stochastic) processes, are measures on the space $(A^\infty,\mathcal F_{A^\infty})$, where $\mathcal F_{A^\infty}$ is the 
Borel sigma-algebra of $A^\infty$. When talking about joint distributions of $N$ samples,  we mean 
distributions on the space $((A^N)^\infty,\mathcal F_{(A^N)^\infty})$.

For each $k,l\in\N$, let $B^{k,l}$ be  the partition of the set $A^k$ into $k$-dimensional cubes  with volume $h_l^k=(1/l)^k$ (the cubes start at 0).
 Moreover, define $B^k=\cup_{l\in\N}B^{k,l}$
and $\mathcal B=\cup_{k=1}^\infty B^k$.
The set $\{B\times A^\infty: B\in B^{k,l}, k,l\in\N\}$
 generates the Borel $\sigma$-algebra on $\R^\infty=A^\infty$.
For a set $B\in \mathcal B$ let $|B|$ be the index $k$ of the set $B^k$  that
$B$ comes from: $|B|= k: B\in B^k$.

We use the abbreviation $X_{1..k}$ for $X_1,\dots,X_k$. 
For a sequence $\x\in A^n$ and a set $B\in \mathcal B$ denote $\nu(\x,B)$
the frequency with which the sequence $\x$ falls in the set~$B$.
{\scriptsize
 \begin{multline*} 
 \nu(\x,B):=\\ \left\{ \begin{array}{rl}  {1\over n-|B|+1}\sum_{i=1}^{n-|B|+1}
 I_{\{(X_i,\dots,X_{i+|B|-1})\in B\}} & \text{ if }n\ge |B|, \\
 0 & \text{ otherwise.}\end{array}\right.
 \end{multline*}
}

A process $\rho$ is {\em stationary}  if 
$
\rho(X_{1..|B|}=B)=\rho(X_{t..t+|B|-1}=B)
$
 for any $B\in A^*$ and $t\in\N$. We further abbreviate   $\rho(B):=\rho(X_{1..|B|}=B)$.
A stationary process $\rho$ is called {\em (stationary) ergodic} if
   the frequency of occurrence of each word
$B$ in a sequence $X_1,X_2,\dots$ generated by $\rho$ tends to its
a priori (or limiting) probability a.s.: 
$
\rho(\lim_{n\rightarrow\infty}\nu(X_{1..n},B)= \rho(B))=1.
$ 
 Denote $\mathcal E$ the set  of all stationary ergodic processes.

\begin{definition}[distributional distance]
 The  distributional distance  is defined for a pair of processes
$\rho_1,\rho_2$ as follows~(e.g. \cite{Gray:88})
$$
d(\rho_1,\rho_2)=\sum_{m,l=1}^\infty w_m w_l \sum_{B\in B^{m,l}} |\rho_1(B)-\rho_2(B)|,
$$
where $w_j=2^{-j}$.
\end{definition}
(The   weights in the definition  are fixed for the sake of concreteness only; we could take any other summable sequence of
positive weights instead.)
In words, we are taking a sum over a series of partitions into cubes of decreasing volume (indexed by $l$) of all sets $A^k$, $k\in\N$, and 
count the differences in probabilities of all cubes in all these partitions. These differences in probabilities are weighted: smaller
weights are given to larger $k$ and finer partitions.
It is easy to see that $d$ is a metric.
We refer to \cite{Gray:88} for more information on this metric  and its properties.

The clustering algorithms presented below are based  
 on {\em empirical estimates of  the  distance $d$}:
\begin{multline}\label{eq:emd}
 \hat d(X^1_{1..n_1},X^2_{1..n_2})=\\\sum_{m,l=1}^\infty w_m w_l \sum_{B\in B^{m,l}} |\nu(X^1_{1..n_1},B)-\nu(X^2_{1..n_2},B)|,
\end{multline}
where $n_1,n_2\in\N$, $\rho\in\S$, $X^i_{1..n_i}\in A^{n_i}$. 

Although the expression~(\ref{eq:emd}) involves taking three infinite sums, it will be shown below that it can be easily calculated.

\begin{lemma}[$\hat d$ is consistent]\label{th:dd} Let $\rho_1,\rho_2\in\mathcal E$ and let two samples $\x_1=X^1_{1..n_1}$
and $\x_2=X^2_{1..n_2}$ be generated by a distribution $\rho$ such that the  marginal distribution of $X^i_{1..n_1}$ is $\rho_i$, $i=1,2$,
and the joint distribution $\rho$ is stationary ergodic. Then
$$
\lim_{n_1,n_2\rightarrow\infty}\hat d(X^1_{1..n_1},X^2_{1..n_2})=d(\rho_1,\rho_2)\ \rho\text{--a.s.}
$$
\end{lemma}
\begin{proof}
The idea of the proof is simple: for each set $B\in\mathcal B$, the frequency with which the sample $\x_1$ falls into $B$ converges to the 
probability $\rho_1(B)$, and analogously for the second sample. When the sample sizes grow, there will be more and more sets $B\in\mathcal B$
whose frequencies have already converged to the probabilities, so that the cumulative weight of those sets whose frequencies
have not converged yet, will tend to 0.

For any $\epsilon>0$ we can find an index $J$ such that
$\sum_{i,j=J}^\infty w_iw_j<\epsilon/3$.
Moreover, for each $m,l$ we can find such elements $B_1^{m,l},\dots,B_{t_{m,l}}^{m,l}$, for some $t_{m,l}\in\N$, of the 
partition $B^{m,l}$ that $\rho_i(\cup_{i=1}^{t_{m,l}}B_i^{m,l})\ge 1-\epsilon/6Jw_mw_l$.
For each $B^{m,l}_j$, where $m,l\le J$ and $j\le t_{m,l}$, we have $\nu((X^1_1,\dots,X^1_{n_1}),B_j^{m,l})\rightarrow \rho_1(B^{m,l}_j)$
a.s., so that
\begin{multline*}
 |\nu((X^1_1,\dots,X^1_{n_1}),B^{m,l}_j) - \rho_1(B^{m,l}_j)|\\<\rho_1(B^{m,l}_j)\epsilon/(6Jw_j)
\end{multline*} for all $n_1\ge u$, for some $u\in\N$; define $U^{m,l}_j:=u$.
 Let
$U:=\max_{m,l\le J, j\le t_{m,l}}U^{m,l}_j$ ($U$ depends
on the realization $X^1_1,X^1_2,\dots$).
Define analogously $V$ for the sequence $(X^2_1,X^2_2,\dots)$. Thus for $n_1>U$ and $n_2>V$ we have
{\scriptsize
\begin{multline*}
   |\hat d(\x_1,\x_2) - d(\rho_1,\rho_2)| =
\\
\left|\sum_{m,l=1}^\infty
w_mw_l\sum_{B\in B^{k,l}}\big(|\nu(\x_1,B)-\nu(\x_2,B)| - |\rho_1(B)-\rho_2(B)| \big)
\right|\\
   \le \sum_{m,l=1}^\infty
w_mw_l\sum_{B\in B^{k,l}} w_i\big(|\nu(\x_1,B)-\rho_1(B)| +
|\nu(\x_2,B)-\rho_2(B)| \big) \\
    \le \sum_{m,l=1}^J w_mw_l\sum_{i=1}^{t_{k,l}}\big(|\nu(\x_1,B^{m,l}_i)-\rho_1(B^{m,l}_i)| \\
 +
|\nu(\x_2,B^{m,l}_i)-\rho_2(B^{m,l}_i)| \big) +2\epsilon/3\\
   \le \sum_{m,l=1}^Jw_mw_l\sum_{i=1}^{t_{k,l}}(\rho_1(B^{m,l}_i)\epsilon/(6Jw_mw_l) \\+ \rho_2(B^{m,l}_i)\epsilon/(6Jw_mw_l))
+2\epsilon/3 \le\epsilon,
\end{multline*}
}
which proves the  statement.
\end{proof}

\section{Main results}
The clustering problem can be defined as follows. We are given $N$ samples $\x_1,\dots,\x_N$, where 
each sample $\x_i$ is a string of length $n_i$ of symbols from $A$: $\x_i=X^i_{1..n_i}$. 
 Each  sample is  generated by one out of $k$ different {\em unknown} stationary ergodic distributions $\rho_1,\dots,\rho_k\in\mathcal E$.
Thus, there is a partitioning $I=\{I_1,\dots,I_k\}$ of  the set $\{1..N\}$ into $k$ {\em disjoint} subsets $I_j, j=1..k$ 
$$
\{1..N\}=\cup_{j=1}^k I_j,
$$
such that $\x_j$, $1\le j\le N$ is generated by $\rho_j$ if and only if $j\in I_j$.
The partitioning $I$ is called the {\em target clustering} and the sets $I_i, 1\le i\le k$, are called the
{\em target clusters}. Given samples $\x_1,\dots,\x_N$ and a target clustering $I$,  let  $I(\x)$
denote the cluster that contains~$\x$.

A {\em clustering function} $F$ takes a finite number of samples  
$\x_1,\dots,\x_N$ and an optional parameter $k$ (the target number of 
clusters) and outputs a partition $F(\x_1,\dots,\x_N,(k))=\{T_1,\dots,T_k\}$ of the set $\{1..N\}$.
\begin{definition}[asymptotic consistency] Let a finite number $N$ of samples be given, and let the target 
clustering partition be $I$. Define  $n=\min\{n_1,\dots,n_N\}$.
 A clustering function  $F$ is strongly asymptotically consistent if  
$
  F(\x_1,\dots,\x_N,(k))=I
$ from some $n$ on with probability~1.
 A clustering function is weakly asymptotically consistent if  
$
  P(F(\x_1,\dots,\x_N,(k))=I)\to1.
$
\end{definition}

Note that the consistency is asymptotic with respect to {\em the minimal length of the sample}, and not with respect to the {\em number of samples}.

\subsection{Known number of clusters}
Algorithm~\ref{alg:1} is a simple clustering algorithm, which, given the number $k$ of clusters, will be shown 
to be consistent under most general assumptions.
It works as follows. The  point $\x_1$ is assigned to the first cluster. 
Next, find the point that is farthest away from $\x_1$ in the empirical distributional distance $\hat d$, and 
assign this point to the second cluster.  For each $j=3..k$, find a point that maximizes the minimal distance
to those points already assigned to clusters, and assign it to the cluster $j$.
Thus we have one point in each of the $k$ clusters. Next simply assign each of the remaining points
to the cluster that contains the closest points from those $k$ already assigned.
(One may notice that  Algorithm~\ref{alg:1} is one iteration of the $k$-means algorithm, with a specific initialization,
and a specially  designed distance.)
\begin{algorithm}
\caption{The case of known number of clusters~$k$}
\label{alg:1}
\begin{algorithmic}
\STATE {INPUT}: The number of clusters $k$, samples $\x_1,\dots,x_N$.
\STATE Initialize: $j:=1$, $c_1 :=1$, $T_1:=\{x_{c_1}\}$.
\FOR {$j:=2$ to $k$}
\STATE $c_j:=\argmax\{i=1,\dots,N: \min_{t=1}^{j-1}\hat d(\x_i,\x_{c_t})\}$
\STATE  $T_j:=\{x_{c_j}\}$
\ENDFOR
\FOR {$i=1$ to $N$}
\STATE Put $\x_i$ into the set $T_{\argmin_{j=1}^k{\hat d(\x_i,\x_{c_j}})}$
\ENDFOR
\STATE OUTPUT: the sets $T_j$, $j=1..k$.
\end{algorithmic}
\end{algorithm}
\begin{proposition}[calculating $\hat d(\x_1,\x_2)$]\label{th:cd}
 For two samples $\x_1=X^1_{1..n_1}$ and $\x_2=X^2_{1..n_2}$ the 
computational complexity (time and space) of calculating the empirical distributional distance $\hat d(\x_1,\x_2)$~(\ref{eq:emd}) is $O(n^2\log s^{-1}_{\min})$,
where $n=\max(n_1,n_2)$ and  
$$
s_{\min}=\min_{i=1..n_1, j=1..n_2, X^1_i\ne X^2_j}|X^1_i-X^2_j|.
$$
\end{proposition}
\begin{proof}
 First,  observe that for fixed $m$ and $l$, the sum 
\begin{equation}\label{eq:psum}
T^{m,l}:=\sum_{B\in B^{m,l}} |\nu(X^1_{1..n_1},B)-\nu(X^2_{1..n_2},B)|
\end{equation}
has not more than $n_1+n_2 -2m+2$  non-zero terms  (assuming $m\le n_1,  n_2$; the other case is obvious).
Indeed, for each $i=0,1$, in the sample $\x_i$  there are $n_i-m+1$ tuples of size $k$: $X^i_{1..m}, X^i_{2..m+1},\dots,X^i_{n_1-m+1..n_1}$.
Therefore, the complexity of calculating  $T^{m,l}$ is $O(n_1+n_2 -2m+2)=O(n)$.
Furthermore, observe that for each $m$, for all $l>\log s^{-1}_{\min}$ the term $T^{m,l}$ is constant.
Therefore, it is enough to calculate $T^{m,1},\dots,T^{m,\log s^{-1}_{\min}}$, since
 for fixed $m$  
{\scriptsize
\begin{equation*}
 \sum_{l=1}^\infty w_mw_l T^{m,l}=w_mw_{\log s^{-1}_{\min}} T^{m,\log s^{-1}_{\min}}+\sum_{l=1}^{\log s^{-1}_{\min}} w_mw_l T^{m,l}
\end{equation*}
}
(that is, we double the weight of the last non-zero term).
Thus, the complexity of calculating $\sum_{l=1}^\infty w_mw_l T^{m,l}$
is $O(n\log s^{-1}_{\min})$. Finally, for all $m>n$ we have $T^{m,l}=0$. Since $\hat d(\x_1,\x_2)=\sum_{m,l=1}^\infty w_m,w_l T^{m,l}$, the statement is proven.
\end{proof}

\begin{theorem}\label{th:cons}
 Let $N\in\N$ and suppose that the samples $\x_1,\dots,\x_N$ are generated in such a way that the joint distribution is stationary ergodic.
If the correct number of clusters $k$ is known, then Algorithm~\ref{alg:1} is strongly asymptotically consistent.
Algorithm~\ref{alg:1} makes $O(kN)$ calculations of $\hat d(\cdot,\cdot)$, so that its computational complexity is 
$O(kN n_{\max}^2\log s_{\min}^{-1})$, where $n_{\max}=\max_{i=1}^kn_i$ and 
$$
s_{\min}=\min_{u,v=1..N, u\ne v, i=1..n_u, j=1..n_v, X^u_i\ne X^v_j}|X^u_i-X^v_j|.
$$
\end{theorem}
 Observe that the samples are not required to be generated independently. The only requirement on the distribution of samples is that the joint 
distribution is stationary ergodic. This is perhaps one of the mildest possible probabilistic assumptions.
\begin{proof}
 By Lemma~\ref{th:dd}, $\hat d(\x_i,\x_j)$, $i,j\in\{1..N\}$ converges to 0 if and only if $\x_i$ and $\x_j$ are in the same
cluster. Since there are only finitely many samples $\x_i$, there exists some $\delta>0$ such that,  from some $n$ on,
we will have $\hat d(\x_i,\x_j)<\delta$ if $\x_i,\x_j$ belong to the same target cluster ($I(\x_i)=I(\x_j)$),
and $\hat d(\x_i,\x_j)>\delta$ otherwise ($I(\x_i)\ne I(\x_j)$).
Therefore, from some $n$ on,   for every $j\le k$ we will have   $\max\{i=1,\dots,N: \min_{t=1}^{j-1}\hat d(\x_i,\x_{c_t})\}>\delta$
and 
the sample $\x_{c_j}$, where $c_j=\argmax\{i=1,\dots,N: \min_{t=1}^{j-1}\hat d(\x_i,\x_{c_t})\}$, will be selected from a target
cluster that does not contain any $\x_{c_i}$, $i<j$. The consistency statement follows.

Next, let us find how many 
pairwise distance estimates $\hat d(\x_i,\x_j)$ the algorithm has to make.
On the first iteration of the loop, it has to calculate $\hat d(\x_i,\x_{c_1})$ for all $i=1..N$.
On the second iteration, it needs again $\hat d(\x_i,\x_{c_1})$ for all $i=1..N$, which are already 
calculated, and also $\hat d(\x_i,\x_{c_2})$ for all $i=1..N$, and so on: on $j$th iteration 
of the loop we need to calculate $ d(\x_i,\x_{c_j})$, $i=1..N$, which gives at most $kN$ pairwise
distance calculations in total. 
The statement about computational complexity follows from this and Proposition~\ref{th:cd}:
indeed, apart from the calculation of $\hat d$, the rest of the computations is
 of order  $O(kN)$.
\end{proof}

\noindent{\bf Complexity--precision trade--off.} The bound on the computational complexity of Algorithm~\ref{alg:1}, given 
in Theorem~\ref{th:cons}, is given for the case of {\em precisely} calculated distance estimates $\hat d(\cdot,\cdot)$.
However, precise estimates are not needed if we only want to have an asymptotically consistent algorithm.
Indeed, following  the proof of Lemma~\ref{th:dd}, it is easy to check that if we replace in~(\ref{eq:emd}) the infinite sums with sums
over any number of terms $m_n$, $l_n$ that grows to infinity with $n=\min(n_1,n_2)$, 
and if we replace partitions $B^{m,l}$ by 
their (finite) subsets $B^{m,l,n}$ which increase to $B^{m,l}$, 
then we still have a consistent
estimate of $d(\cdot,\cdot)$.

\begin{definition}[$\check d$]
Let $m_n, l_n$ be some sequences of numbers,   $B^{m,l,n}\subset B^{m,l}$ for all $m,l,n\in\N$, and denote  $n:=\min\{n_1,n_2\}$. Define
\begin{multline}\label{eq:emd3}
 \check d(X^1_{1..n_1},X^2_{1..n_2}):=\\\sum_{m=1}^{m_n}\sum_{l=1}^{l_n} w_m w_l \sum_{B\in B^{m,l,n}} |\nu(X^1_{1..n_1},B)-\nu(X^2_{1..n_2},B)|.
\end{multline}
\end{definition}
\begin{lemma}[$\check d$ is consistent]\label{th:td} 
Assume the conditions of Lemma~\ref{th:dd}. 
Let $l_n$ and $m_n$ be any sequences of integers that go to infinity with $n$, 
and let, for each $m,l\in\N$, the sets $B^{m,l,n}$, $n\in\N$ be an increasing sequence of  subsets of $B^{m,l}$,   such that $\cup_{n\in\N} B^{m,l,n}=B^{m,l}$.
 Then
$$
\lim_{n_1,n_2\rightarrow\infty}\check d(X^1_{1..n_1},X^2_{1..n_2})=d(\rho_1,\rho_2)\ \rho\text{--a.s.}.
$$
\end{lemma}
\begin{proof}
 It is enough to observe that 
\begin{multline*}
\lim_{n_1,n_2\to\infty}\sum_{m=1}^{m_n}\sum_{l=1}^{l_n} w_m w_l \sum_{B\in B^{m,l,n}} |\rho_1(B)-\rho_2(B)|\\=d(\rho_1,\rho_2),
\end{multline*}
and then follow the proof of Lemma~\ref{th:dd}.
\end{proof}

 If we use the estimate $\check d(\cdot,\cdot)$ in Algorithm~\ref{alg:1} (instead of $\hat d(\cdot,\cdot)$), then we still get an asymptotically consistent clustering function.
Thus the following statement holds true.
\begin{proposition}\label{th:as2}
 Assume the conditions of Theorem~\ref{th:cons}. 
For all sequences  $m_n, l_n$ of numbers that increase to infinity with $n$, there is a strongly asymptotically consistent clustering
algorithm, whose computational complexity is $O(kN n_{\max}m_{n_{\max}}l_{n_{\max}})$.
\end{proposition}
On the one hand, Proposition~\ref{th:as2} can be thought of as an artifact of the asymptotic definition of consistency;
on the other hand, in practice precise calculation of $\hat d(\cdot,\cdot)$ is hardly necessary. What we get from Proposition~\ref{th:as2}
is the possibility to select the appropriate trade--off between the computational burden, and the precision of clustering before asymptotic.

Note that the bound in Proposition~\ref{th:as2} does not involve the sizes of the sets $B^{m,l,n}$; in particular, one can
take $B^{m,l,n}=B^{m,l}$ for all $n$.
This is because, for every two samples $X_{1..n}^1$ and $X_{1..n}^2$, this sum 
has no more than $2n$ non-zero terms, whatever  are $m,l$. However, in the following section, where we are after
clustering with an unknown number of clusters $k$, and thus after controlled rates of convergence, 
the sizes of the sets $B^{m,l,n}$ will appear in the bounds.

\subsection{Unknown number of clusters}
So far we have shown that when the number of clusters is known in advance, consistent clustering is possible 
under the only assumption that the joint distribution of the samples is stationary ergodic.
However, under this assumption, in general, consistent clustering with unknown number of clusters is impossible.
Indeed, as was shown in \cite{Ryabko:10discr},  when we have only two {\em binary-valued} samples, generated {\em independently}
by two stationary ergodic distributions,
 it is impossible to decide whether they have been generated by the same or by different distributions, even in the 
sense of weak asymptotic consistency (this holds even if the distributions come from a smaller class: the set of all $B$-processes). 
Therefore, if the number of clusters is unknown, we have to settle for less, which means that we have to make 
stronger assumptions on the data. What we need is  known rates of convergence of frequencies to 
their expectations. Such rates are provided by assumptions on the mixing rates of the distribution 
generating the data. Here we will show that under rather mild assumptions on the mixing rates
(and, again,  without any modeling assumptions or assumptions of independence), consistent
clustering is possible when the number of clusters is unknown.

In this section we assume that all the samples are $[0,1]$-valued (that is, $X_i^j\in[0,1]$); extension to arbitrary bounded (multidimensional)
ranges is straightforward. Next we introduce {\em mixing coefficients}, mainly following \cite{Bosq:96} in formulations.
Informally, mixing coefficients of a stochastic process measure how fast the process forgets about its past.
Any one-way infinite stationary process $X_1,X_2,\dots$ can be extended backwards to make a two-way infinite process $\dots,X_{-1},X_0,X_1,\dots$
with the same distribution. In the definition below we assume such an extension.
Define the $\alpha$ 
mixing coefficients as
\begin{multline}\label{eq:alph}
 \alpha(n)=\sup_{A\in\sigma(\dots,X_{-1},X_0),B\in\sigma(X_{n},X_{n+1},\dots))}\\|P(A\cap B)-P(A)P(B)|,
\end{multline}
where $\sigma(..)$ stays for the sigma-algebra generated by random variables in brackets.
These coefficients are non-increasing. 
A process is called  strongly {\em $\alpha$-mixing} if $\alpha(n)\to0$.
Many important classes of processes satisfy the mixing conditions. For example, if a process is a 
 stationary  irreducible aperiodic Hidden Markov process, then it is $\alpha$-mixing.
 If the underlying Markov chain is finite-state, then the coefficients decrease 
exponentially fast. Other probabilistic assumptions  can be used to obtain bounds on the mixing coefficients,
see e.g. \cite{Bradley:05} and references therein.


{\bf Algorithm~2} is very simple. Its inputs are: samples $\x_1,\dots,x_N$;  the threshold level $\delta\in(0,1)$, the parameters $m,l\in\N$, $B^{m,l,n}$.
The algorithm assigns to the same cluster all samples which are at most $\delta$-far
from each other, as measured by  $\check d(\cdot,\cdot)$. 
The estimate $\check d(\cdot,\cdot)$ can be calculated in the same way as $\hat d(\cdot,\cdot)$ (see Proposition~\ref{th:cd} and its proof).
We do not give a pseudo code implementation of this algorithm, since it's rather obvious.

The idea is that the threshold level $\delta$ is selected according to the minimal  length of a sample and the (known bounds on) mixing 
rates of the process $\rho$ generating the samples (see Theorem~\ref{th:cons3}). 
%
%

The next theorem shows that, if the joint distribution of the samples satisfies $\alpha(n)\le\alpha_n\to0$, where 
$\alpha_n$ are known, then one can select (based on $\alpha_n$ only) the parameters of Algorithm~2 in such a way that it is weakly asymptotically consistent. 
Moreover, a bound on the probability of error before asymptotic is provided.
\begin{theorem}[Algorithm~2 is consistent, unknown $k$]\label{th:cons3}
Fix sequences $\alpha_n\in(0,1)$, $m_n,l_n,b_n\in\N$, and let  $B^{m,l,n}\subset B^{m,l}$ be an increasing sequence of finite sets, for each $m,l\in\N$.
 Set $b_n:=\max_{l\le l_n,m\le m_n}|B^{m,l,n}|$.
Let also  $\delta_n\in(0,1)$. Let $N\in\N$ and suppose that the samples $\x_1,\dots,\x_N$ are generated in such a way that the (unknown) joint distribution
$\rho$ is stationary ergodic, 
and satisfies $\alpha_n(\rho)\le\alpha_n$, for all $n\in\N$. 
Then for every sequence $q_n\in[0..n/2]$,  Algorithm~2, with the above parameters, satisfies
\begin{equation}\label{eq:alg2}
 \rho(T\ne I)\le  2N(N+1)(m_nl_nb_n\gamma_n(\delta_n)+\gamma_n(\epsilon_\rho)) 
\end{equation}
where 
$$
  \gamma(\delta)=(2e^{-q_n\delta^2/32}+11(1+4/\delta)^{1/2}q_n\alpha_{(n-2m_n)/2q_n}),
$$
 $T$ is the partition output by the algorithm, $I$ is the target clustering, $\epsilon_\rho$ is a constant
that depends only on $\rho$, and $n=\min_{i=1..N}n_i$.

In particular, if $\alpha_n=o(1)$, then, selecting the parameters  in such a way that 
$\delta_n= o(1)$, $q_n,m_n,l_n,b_n=o(n)$, $q_n,m_n,l_n\to\infty$, $\cup_{k\in\N}B^{m,l,k}=B^{m,l}$,  $b^{m,l}_n\to\infty$, for all $m,l\in\N$, and, finally,
$$
m_n l_n b_n(e^{-q_n\delta^2_n}+\delta_n^{-1/2}q_n\alpha_{(n-2m_n)/2q_n})=o(1),
$$
as is always possible, {\em Algorithm~2 is weakly asymptotically consistent} (with the number of clusters $k$ unknown). 
The computational complexity of Algorithm~2 is $O(N^2m_{n_{\max}}l_{n_{\max}}b_{n_{\max}})$, and is bounded by $O(N^2n_{\max}^2\log s^{-1})$, where
$n_{\max}$ and $\log s^{-1}_{\min}$ are defined as in Theorem~\ref{th:cons}.
\end{theorem}
\begin{proof}
 We use the following bound from  \cite{Bosq:96}: for any zero-mean random 
process $Y_1,Y_2,\dots$, every $n\in\N$ and  every  $q\in[1..n/2]$ we have
\begin{multline*}
 P\left(|\sum_{i=1}^nY_i|>n\epsilon\right)\\\le 4\exp(-q\epsilon^2/8)+22(1+4/\epsilon)^{1/2}q\alpha(n/2q).
\end{multline*}
For every $j=1..N$, every $m<n$, $l\in\N$, and  $B\in B^{m,l}$,  define the processes $Y^j_1,Y^j_2,\dots$, where 
$$Y^j_t:=\mathbb I_{(X^j_t,\dots,X^j_{t+m-1})\in B} -\rho(X^j_{1..m}\in B).$$ It is easy to see that $\alpha$-mixing 
coefficients for this process satisfy $\alpha(n)\le \alpha_{n-2m}$. Thus, 
\begin{multline}\label{eq:lots}
 \rho(|\nu(X^j_{1..n_j},B)-\rho(X^j_{1..m}\in B)|>\epsilon/2)\le\gamma_n(\epsilon)
\end{multline}
Then for every $i,j\in[1..N]$ such that $I(\x_i)=I(\x_j)$ (that is, $\x_i$ and $\x_j$ are in the same cluster) we have
$$
\rho(|\nu(X^i_{1..n_i},B)-\nu(X^j_{1..n_j},B)|>\epsilon)\le 2\gamma_n(\epsilon).
$$
Using the union bound, summing over $m,l,$ and $B$, we obtain
\begin{equation}\label{eq:l2}
\rho(\check d(\x_i,\x_j)>\epsilon)\le 2m_nl_nb_n \gamma_n(\epsilon).
\end{equation}
Next, let $i,j$ be such that $I(\x_i)\ne I(\x_j)$. Then, for some $m_{i,j},l_{i,j}\in\N$ there is $B_{i,j}\in B^{m_{i,j},l_{i,j}}$ such 
that $|\rho(X^i_{1..|B_{i,j}|}\in B_{i,j}) - \rho(X^j_{1..|B_{i,j}|}\in B_{i,j})|>2\tau_{i,j}$ for some $\tau_{i,j}>0$.
Then for every $\epsilon<\tau_{i,j}/2$ we have
\begin{multline}\label{eq:ll2}
\rho(|\nu(X^i_{1..n_i},B_{i,j})-\nu(X^j_{1..n_j},B_{i,j})|<\epsilon)\le \\
 \rho(|\nu(X^i_{1..n_i},B_{i,j})- \rho(X^i_{1..|B|}\in B_{i,j})|>\tau_{i,j}) \\+\rho(|\nu(X^j_{1..n_j},B_{i,j})- \rho(X^j_{1..|B_{i,j}|}\in B_{i,j})|>\tau_{i,j})
\\\le 2\gamma_n(\tau_{i,j}).
\end{multline}
Moreover, for $\epsilon< w_{m_{i,j}}w_{l_{i,j}}\tau_{i,j}/2$
\begin{equation}\label{eq:l3}
  \rho(\check d(\x_i,\x_j)>\epsilon)\le 2\gamma_n(w_{m_{i,j}}w_{l_{i,j}}\tau_{i,j}).
\end{equation}
Define
$
 \epsilon_\rho:=\min_{i,j=1..N: I(\x_i)\ne I(\x_j)}w_{m_{i,j}}w_{l_{i,j}}\tau_{i,j}/2.
$
Clearly, from 
this  and~(\ref{eq:ll2}), for every $\epsilon<2\epsilon_\rho$ we obtain
\begin{equation}\label{eq:l4}
  \rho(\check d(\x_i,\x_j)>\epsilon)\le 2\gamma_n(\epsilon_\rho).
\end{equation}

If, for every pair $i,j$ of samples, $\check d(\x_i,\x_j)<\delta_n$ if and only if $I(\x_i)=I(\x_j)$, then 
Algorithm~2 gives a correct answer. Therefore, taking the bounds~(\ref{eq:l2}) and~(\ref{eq:l4}) 
together for each of the $N(N+1)/2$ pairs of samples, we obtain~(\ref{eq:alg2}). The complexity statement can be established
analogously  to that in~Theorem~\ref{th:cons}.
\end{proof}

While Theorem~\ref{th:cons3} shows that $\alpha$-mixing with a  known bound on the coefficients 
is sufficient to achieve asymptotic consistency, the bound~(\ref{eq:alg2}) on the 
probability of error includes as multiplicative terms all the parameters $m_n$, $l_n$ and $b_n$ of the algorithm,
which can make it large for practically useful choices of the parameters. 
The multiplicative factors are due to the fact that we take a bound on the divergence of  each individual 
frequency of each cell of each partition from its expectation, and then take a union 
bound over all of these. To obtain a more realistic performance guarantee, we would like to have a bound 
on the divergence of {\em all} the frequencies of all cells of a given partition from their expectations.
Such uniform divergence estimates are possible under stronger assumptions; namely, they can be established
under some  assumptions on $\beta$-mixing coefficients, which are defined as follows 
\begin{multline*}
 \beta(n)=\E\sup_{B\in\sigma(X_{n},\dots))} |P(B)-P(B|\sigma(\dots,X_0))|.
\end{multline*}
These coefficients satisfy $2\alpha(n)\le \beta(n)$ (see e.g. \cite{Bosq:96}), so assumptions on the speed of decrease of $\beta$-coefficients are stronger.
Using the uniform bounds given in \cite{Karandikar:02}, one can obtain 
a statement similar
  to that in Theorem~\ref{th:cons3}, with $\alpha$-mixing replaced
by $\beta$-mixing, and without the multiplicative factor $b_n$.

\section{Conclusion}
We have proposed a framework for defining consistency of clustering algorithms, when 
the data comes as a set of samples drawn from stationary processes. The main advantage of 
this framework is its generality: no assumptions have to be made on the distribution
of the data, beyond stationarity and ergodicity. The proposed notion of consistency is
so simple and natural, that it may be suggested to be used as a basic sanity-check for all 
clustering algorithms that are used on sequence-like data. For example, it is  easy to see that the $k$-means algorithm 
will be consistent  with some initializations (e.g. with the one used in Algorithm~\ref{alg:1}) but not 
with others  (e.g. not with the random one). 

While the algorithms that we presented to demonstrate the existence of consistent clustering 
methods are computationally efficient and easy to implement, the main value of the established
results is theoretical. 
 As it was mentioned in the introduction, it can be suggested that
 for practical applications empirical estimates of the distributional distance can be replaced
 with distances based on data compression, in the spirit of \cite{BRyabko:06a, Cilibrasi:05, BRyabko:09}.

Another direction for future research concerns optimal bounds on the speed of convergence: while we show 
that such bounds can be obtained (of course, only in the case of known mixing rates), finding practical 
and tight bounds, for different notions  of mixing rates, remains open. 

Finally, here  we have only considered the setting in which the number $N$ of samples is fixed, while 
the asymptotic is with respect to the lengths of the samples. For on-line clustering problems,
it would be interesting to consider the formulation where both $N$   and the lengths of the samples grow.


\end{document}